\documentclass{article}
\usepackage{tikz-cd}
\usepackage{amsmath}
\usepackage{amsthm}
\usepackage{amssymb}
\usepackage{stmaryrd}
\usepackage{stackengine}
\usepackage{tikz}
\usetikzlibrary{shapes,arrows,positioning,calc}
\usepackage{adjustbox}
\usepackage{pgfplots}
\usepackage{verbatim}
\usepackage{wrapfig}
\usepackage{pifont}
\usepackage{subcaption}

\newtheorem*{theorem}{\bf Theorem}
\newtheorem*{lemma}{\bf Lemma}
\newtheorem*{corollary}{\bf Corollary}
\theoremstyle{definition}
\newtheorem{definition}{\bf Definition}[section]
\newtheorem{example}{\bf Example}[section]
\newtheorem*{conjecture}{\bf Conjecture}
\newtheorem{OpenProblem}{\bf Open Problem}
\newcommand{\ReLU}{\textsc{ReLU}}
\newcommand{\argmin}{\mathop{\mathrm{argmin}}} 

\begin{document}

\usetikzlibrary{arrows,shapes,circuits.logic.US,shapes.gates.logic.IEC,calc,decorations.markings,patterns}
\tikzstyle{branch}=[fill,shape=circle,minimum size=3pt,inner sep=0pt]

\tikzset{
block/.style = {draw, fill=white, rectangle, minimum height=3em, minimum width=3em},
tmp/.style  = {coordinate}, 
sum/.style= {draw, fill=white, circle, node distance=1cm},
input/.style = {coordinate},
output/.style= {coordinate},
pinstyle/.style = {pin edge={to-,thin,black}
}
}

\title{Rethinking Arithmetic for Deep Neural Networks}

\author{
G. A. Constantinides {\tt (g.constantinides@imperial.ac.uk)}}

\date{September 2019}

\maketitle

\begin{abstract}  
We consider efficiency in the implementation of deep neural networks. Hardware accelerators are gaining interest as machine learning becomes one of the drivers of high-performance computing. In these accelerators, the directed graph describing a neural network can be implemented as a directed 
graph describing a Boolean circuit. We make 
this observation precise,  leading naturally to an understanding of practical neural networks as discrete functions, and show that so-called {\em binarised neural networks} are functionally complete. 
In general, our results suggest that it is valuable to consider {\em Boolean circuits as neural networks}, leading to 
the question of which circuit topologies are promising. We argue that continuity is central to generalisation in learning, explore the interaction between data coding, network topology, and node functionality for continuity, and pose some open questions for future research. 
As a first step to bridging the gap between continuous and Boolean views of neural network accelerators, we present some recent results from our work on LUTNet, a novel Field-Programmable Gate Array inference approach. 
Finally, we conclude with additional possible fruitful avenues for research bridging the continuous and discrete views of neural networks.
\end{abstract}


\section*{Notation}
${\mathbb R}$ denotes the reals, and ${\mathbb B} = \{\bot, \top\}$ the set of Boolean truth values, where $\bot$ denotes false and $\top$ denotes true. $\ReLU : {\mathbb R} \to {\mathbb R}$ is used to denote the {\em rectified linear unit} function $x \mapsto \max(0, x)$. $\sigma : {\mathbb R} \to {\mathbb R}$ denotes the {\em sigmoid} function $x \mapsto \frac{2}{1+\exp(-x)} - 1$.
We denote function composition by $\circ$.  ${\mathcal B}_K$ denotes the set of all functions from ${\mathbb B}^K$ to ${\mathbb B}$.
The set of integers is denoted by ${\mathbb Z}$, and the set of integers bounded in absolute value $n$ by ${\mathbb Z}_n = \{ i \in {\mathbb Z} | -n \leq i \leq n \}$. The following Boolean connectives are used: $\neg$ denotes negation, $\wedge$ denotes conjunction, $\vee$ denotes disjunction, and $\oplus$ denotes exclusive or (XOR).

\section{Introduction}
\label{introduction}

This paper considers the development of deep neural networks in the {\em supervised learning} setting~\cite{Goodfellow-et-al-2016}. Inspired by the recent rise of interest in specialised hardware accelerators for deep neural networks~\cite{Erwei2019}, we shall take a fresh look at the question of suitable
network topologies and basic node functionalities for such accelerators. 

We shall begin by defining the supervised learning problem. Let ${\mathbb X}$ denote the set of possible inputs to a machine learning inference function, and ${\mathbb Y}$ denote the set of possible outputs. Imagine that we have an oracle function $r: {\mathbb X} \to {\mathbb Y}$, mapping 
every possible input to the corresponding {\em ideal output} $y = r(x)$. Generally, we will be interested in inference via a family of parametrically-defined functions $f(p;x)$, with parameters drawn from some set ${\mathbb P}$. We will 
often write $f_p(x)$ when we wish to consider the case where the parameter value $p$
has been fixed. These functions will not, in general, produce the ideal output for all possible inputs, 
and therefore we need to consider some notion of inaccuracy, or `loss', $\ell$, which
measures the difference between the ideal output and the actually computed output as $\ell( f_p(x), r(x) )$. For simplicity, we assume in this article that $\ell$ is a {\em metric}~\cite{Searcoid2007} defined on ${\mathbb Y}$. We are generally interested in average-case behaviour of these parametric functions `in the wild', 
on any data that may frequently appear as input in real usage. Lifting the metric on ${\mathbb Y}$ to the following metric defined on functions ${\mathbb X} \to {\mathbb Y}$,

\begin{equation}
\label{idealmetric}
m( f, f' ) = {\mathbb E} \left\{ \ell( f, f' ) \right\},
\end{equation}

\noindent where the expectation is over the input space, we can then pose the question of supervised training as the following optimisation problem of selecting parameters to minimize distance to an oracle function:

\begin{equation}
\argmin\limits_{p \in {\mathbb P}} m(f_p, r)
\end{equation}

There are some practical problems, however. Firstly, it is unlikely that we have access to or knowledge of the distribution of ${\mathbb X}$ or to an oracle function $r$, except through a finite set of samples, known as the {\em training set}. Secondly, as Scheinberg notes~\cite{Scheinberg2016},
the loss function $\ell$ desired in practice ({\em e.g.} an indicator function) may give rise to a computationally intractable optimisation problem. As a result, it is common to aim instead to solve the training problem,

\begin{equation}
p^* = \argmin\limits_{p \in {\mathbb P}} \frac{1}{n} \sum_{i=1}^n {\ell'( f_p(x_i), y_i )}
\end{equation}

\noindent where $(x_i,y_i)$ are the {\em training data} -- inputs for which the ideal output is known -- and $\ell'$ is some suitable, often convex, loss function.

The actual accuracy of the resulting function $f_{p^*}$, can then be evaluated on some other set of data $(x'_i,y'_i)$-- the {\em  test data}, as a proxy for $m( f_{p^*}, r )$, to obtain the {\em test error}:

\begin{equation}
\frac{1}{n'} \sum_{i=1}^{n'} {\ell( f_{p^*}(x'_i), y'_i )}
\end{equation}

It turns out that this setting therefore imposes particular restrictions on the family of parameterised functions $f$, because we wish $p^*$ -- which was selected based only on the training data -- to also work well for the test data, as well as ensuring several other properties to be discussed. This fundamental problem:
the design of families of parameterised functions for this purpose, is the key subject of study of this paper. In particular, we address here the case where the functions $f_p$ map from one finite set to another, which is always the practical setting in a finite-precision computer. By considering
the discrete problem explicitly, several new insights are developed, which may be of value to those researching highly-efficient machine inference. 


The structure of this paper is as follows:


Section~\ref{DeepLearning} introduces a model of computation defined by typed graphs. We use this model to develop a deeper understanding of the computation of inference functions in deep neural networks, discussing suitable choices for such functions. The model also lets us reason about
their approximation by discrete functions, and hence the potential for hardware implementations of such computations. 
In Section~\ref{DiscreteFunctions}, we present an abstract view of the typical digital design process for hardware accelerators of numerical functions. 
We then show that a known family of extremely quantised neural networks is {\em functionally complete}. This result runs counter to standard thinking in hardware-accelerated neural networks, and we shall consider the reason for this apparent contradiction.
Section~\ref{WhichFunctions} revisits the question of appropriate inference functions from Section~\ref{DeepLearning}, but now in the discrete setting. We argue for a trinity of topology, node functionality and metrics as interacting to determine efficient inference computation, and pose some open questions
regarding the extent to which these factors can be decoupled. 
In Section~\ref{TwoApproaches}, we consider an approach for efficient FPGA inference, known as LUTNet, recently published by my research group as an example of initial work bridging the continuous and discrete setting. 
Finally, Section~\ref{Conclusions} draws conclusions and points to several fruitful avenues for further research.


\section{Networks and Inference Functions}
\label{DeepLearning}

\noindent {\bf A Graphical Approach}

A graphical approach is universally used to describe -- formally or informally -- the computations performed by deep neural networks. In this section we shall develop a slightly unorthodox but very general formalism, which will be of use throughout the paper. Our aim here is to distinguish the {\em syntactic} 
description of neural networks as graphs from the {\em semantic} interpretation as functions. This distinction will be important because the transformations applied to develop a realisation of a neural network as a program or a piece of digital hardware are primarily based on the syntactic representation.

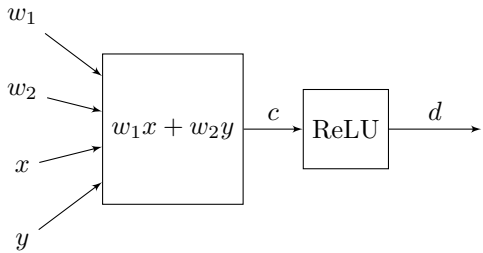
\begin{figure}
\begin{center}
\begin{tikzpicture}[auto, scale=2.0, node distance=2cm,>=latex']
    \node [name=x] (x) at (1,-1) {$x$};
    \node [name=y] (y) at (1,-1.5) {$y$};
    \node [name=w1] (w1) at (1,0) {$w_1$};
    \node [name=w2] at (1,-0.5) (w2) {$w_2$};
    \node [block, minimum height=2cm] (dotprod) at (2,-0.75) {$w_1 x + w_2 y$};
    \node [block, right of=dotprod, node distance=2.3cm] (relu){$\text{ReLU}$};
    \node [output, right of=relu, node distance=1.8cm] (d) {};
    \draw [->] (w1) -- (dotprod);
    \draw [->] (w2) --  (dotprod);
    \draw [->] (x) -- (dotprod);
    \draw [->] (y) -- (dotprod);
    \draw [->] (dotprod) --node{$c$} (relu);
    \draw [->] (relu) --node{$d$} (d);
 \end{tikzpicture}
 \end{center}
 \caption{A simple network consisting of two vertices. One vertex has two parameters and two activation inputs and has function $(w_1, w_2; x, y) \mapsto w_1 x + w_2 y$. The other vertex has no parameters and one activation input and has function $c \mapsto \ReLU(c)$.\label{network}}
\end{figure}

\begin{definition}
An {\bf edge} $e$ is simply a unique label, together with a set such as ${\mathbb R}$ or ${\mathbb F}$, which can be interpreted as the type of data carried by the edge in a network.
\end{definition}

\begin{example}
In Fig.~\ref{network}, $x$, $y$, $w_1$, $w_2$, $c$ and $d$ are all edges, which we will take to be of real type.
\end{example}

\begin{definition}
A {\bf vertex} $v$ is a tuple $v = (\textsc{param},\textsc{in},\textsc{out},\textsc{func})$ of ordered lists of edges $\textsc{param}$, $\textsc{in}$ and $\textsc{out}$, together with a function $\textsc{func}$ from the Cartesian product of the sets defined by $\textsc{param}$ and $\textsc{in}$ to that of those defined by $\textsc{out}$.
\end{definition}

\begin{example}
In Fig.~\ref{network}, there is a vertex $( (w_1, w_2), (x, y), c, ( w_1, w_2; x, y ) \mapsto w_1x + w_2y)$. 
\end{example}

The purpose of distinguishing $\textsc{param}$ from $\textsc{in}$ is to identify those values (parameters) that are intended to be determined once, offline, versus those values (activations) that are intended to be 
change each time the graph is used in inference -- we use a semicolon to separate parameters from activations, for readability purposes. 
In this example, $w_1$ and $w_2$ are parameters, commonly called {\em weights}, while $x$ and $y$ are not. There is one other vertex $( (), c, d,\ReLU )$ shown in the figure; this vertex has an empty parameter list.

\begin{definition}
A {\bf network} $N$ is a pair $N = (V,\textsc{priout})$ of vertices together with a distinguished set of edges $\textsc{priout}$ such that:
\begin{itemize}
\item No edge appearing in the $\textsc{param}$ list of any vertex also appears in the $\textsc{out}$ list of any vertex. 
\item No edge appears in the $\textsc{out}$ list of more than one vertex.
\item All edges in $\textsc{priout}$ appear in the list $\textsc{out}$ list of exactly one vertex.
\end{itemize}

We will refer to {\em parameters of the network} to mean the list of all edges appearing in the $\textsc{param}$ list of any vertex;  {\em inputs to the network} to mean the list of all edges appearing in the $\textsc{in}$ list of some vertex but not in the $\textsc{out}$ list of any vertex; and  
{\em outputs of the network} to be the set $\textsc{priout}$. We will often refer to the `leaf functions' of a network, meaning the collection of functions $\textsc{func}$ of all the vertices in the network.

\end{definition}

\begin{example}
The network shown in Fig.~\ref{network} consists of the two vertices previously described, together with a set of primary outputs. One possibility for such a set is $\textsc{priout} = \{d\}$, but there are other choices, depending on which edges are required to be observable at the network output. The parameters
of this network are $w_1$ and $w_2$, and the inputs to the network are $x$ and $y$.
\end{example}

\begin{definition}
We say that a network $N$ {\bf implements} a function $\llbracket N \rrbracket$ defined through the natural function composition of the individual vertex functions, {\em i.e.} $\llbracket N \rrbracket$ is a function from the Cartesian products of parameters and inputs of the network to the Cartesian product of the outputs of the network,
defined inductively, with vertex functions as the leaf functions.
\end{definition}

\begin{example}
For the network $N$ shown in Fig.~\ref{network}, $\llbracket N \rrbracket = (w_1,w_2; x,y) \mapsto \ReLU(w_1 x + w_2 y)$.
\end{example}


For simplicity, we will consider computations corresponding to acyclic networks -- including the very significant class of Convolutional Neural Networks~\cite{LeCun1989} -- however the formalism can easily be extended to cyclic networks (e.g. LSTMs~\cite{doi:10.1162/neco.1997.9.8.1735}) by lifting computation over the types illustrated above to computations over {\em streams} of those types~\cite{Kahn1974}. This generalisation does not affect the following material. Equally, it is trivial to make networks hierarchical by generalising functions computed to also allow sub-networks, but this will not be required in the sequel.


\vspace{1cm}

\noindent {\bf Functions for Inference}

What kind of functions $f = \llbracket N \rrbracket$ form good candidates for machine learning? And what basic functionality should be implemented by nodes in a network $N$ for this purpose? In practical terms, for deep learning today, the most common leaf functions are inner products, $\ReLU$, sigmoid, and softmax~\cite{Goodfellow-et-al-2016}. However, it is worth considering the various factors that determine this choice now and in the future. Informally, functions should:

\begin{itemize}
\setlength\itemsep{0.2em}
\item[\ding{182}] Generalise well: once the parameter $p$ is selected based on training data, $f_p(x)$ should also tend to perform well over unseen test data.
\item[\ding{183}] Be cheap to compute: the cost (speed, energy) of evaluating the function at inference time should be low.
\item[\ding{184}] Be sufficiently general / expressive: the functions should be capable of approximating a wide variety of oracle functions $r$.
\item[\ding{185}] Be easy to learn: optimisation algorithms used to address the training problem described in Section~\ref{introduction} should be both cheap to execute and also rarely give rise to values of parameter that are grossly suboptimal with respect to the training set.
\end{itemize}

Strang~\cite{Strang2018} argues that continuous piecewise linear (CPL) functions have tended to perform well, explaining the importance of inner product and $\ReLU$ functions in today's networks, as CPL functions are precisely those that are implemented by networks with these vertices. 
Strang argues that {\em continuity} is key to generalisation, which intuitively makes sense: if an untrained input is very close to a trained one, it seems reasonable to expect the corresponding outputs of the network to be very close in turn. 

To make this intuition precise requires us to equip the input and output sets with {\em metrics}, $d$ and $e$, respectively, allowing us to define what it means for inputs and outputs to be `close'. We can then consider the inference function $f_p$ as a function from an input metric space $({\mathbb X},d)$ to an output metric space 
$({\mathbb Y},e)$. We have a choice of options to define continuity; we shall use Lipschitz continuity~\cite{Searcoid2007}, for reasons that will become apparent in the next section. 

\begin{definition}
Suppose $f : {\mathbb X} \to {\mathbb Y}$, where ${\mathbb X}$ is equipped with a metric $d$ and ${\mathbb Y}$ is equipped with a metric $e$. Let $k \in {\mathbb R}$. The function $f$ is $k$-{\bf Lipschitz} if for all $a,b \in {\mathbb X}$, $e( f(a), f(b) ) \leq k d(a,b)$.
\end{definition}

\begin{definition}
A function $f$ is {\bf Lipschitz} if it is $k$-Lipschitz for some $k$.
\end{definition}

\begin{example}
For computation over ${\mathbb R}^n$ with metrics determined by a suitable norm in that space, the $\ReLU$ function is Lipschitz and inner products are Lipschitz,
and thus by composition, networks constructed from these two functions are Lipschitz~\cite{Searcoid2007} and therefore good candidates for generalising beyond training data.
\end{example}

Whether inner products and $\ReLU$ functions are cheap to compute (Property~\ding{183}) depends upon our model of computation; in the abstract Blum-Shub-Smale model for real computation, this is certainly the case~\cite{Blum1989}. It is now well-known that a wide variety of neural networks, including those implementing
CPL functions are {\em universal approximators}, and hence sufficiently general~\cite{Hornik1991,Neal1994} (Property~\ding{184}). This leaves the question of whether such functions are `easy to learn' (Property~\ding{185}). This is still an active area of research,
however theoretical insights such as~\cite{DBLP:journals/corr/abs-1810-02054} combined with practical experience suggest that this is indeed the case.

So while CPL functions over the reals appear to be very promising, practical computers do not compute over the reals. In practice, finite precision datatypes are (almost) always used to approximate computation over the reals, and the picture of appropriate inference functions has the potential to change considerably in this 
setting. We examine this question in this next section.

\section{Discrete Inference}
\label{DiscreteFunctions}

We shall refer to a network where the types of all activations are ${\mathbb R}$ as a {\em real network}, where the types of all activations are ${\mathbb F} \subset {\mathbb R}$ for finite ${\mathbb F}$ as a {\em finite-precision network}, and where the types of all activations are ${\mathbb B}$ as a {\em Boolean network}.
Boolean networks correspond exactly to combinational digital circuits, and so hold a special place from an implementation perspective.

Figure~\ref{digital-design} illustrates the standard digital design process for development of a Boolean network approximating a given real network $G_1$. The first step is that of {\em quantisation}. Here, real data types associated with edges in $G_1$ are replaced by finite precision data types ${\mathbb F}$.
Typical examples are single-precision IEEE floating point arithmetic~\cite{IEEE-standard} as well as various fixed-point arithmetics. Consequently, the functions $\textsc{func}$ performed by each node in the network must also be quantised, hence it is common to require $G_1$'s node functions to be drawn from
a basic set of operators for which this function quantisation process can be performed automatically or is defined by some standard as, {\em e.g.}~$\{ *, - , +, / \}$ are for IEEE floating-point arithmetic. The quantisation process induces a change in function: 
$\iota^{m}_1 \circ \llbracket G_2 \rrbracket \neq \llbracket G_1 \rrbracket \circ \iota^{n}_1$ in general,
and so has been the subject of a considerable amount of work in the DNN literature, with modern machine inference architectures often offering choices of precision that trade performance for accuracy of computation~\cite{Erwei2019},
{\em e.g.}~\cite{IMGNNA}. The main distinguishing features of this setting compared to classical finite precision quantisation results~\cite{Higham2002} are due to the metric $m$ introduced in Section~\ref{introduction}: both its inherently stochastic nature
and distance to an oracle $r$ rather than distance to the underlying real function being the primary concern, {\em i.e.} the ideal quantisation is one that by selecting $\tilde{p}$ minimises $m( \llbracket G_2 \rrbracket_{\tilde{p}}, r)$ rather than $m( \llbracket G_2 \rrbracket_{\tilde{p}}, \llbracket G_1 \rrbracket_p )$. In practice, however,
it is typical to initially select the each element of the quantised parameter independently, effectively relying on repeated application of the triangle inequality applied syntactically to the graph to ensure $m( \llbracket G_2 \rrbracket_{\tilde{p}}, \llbracket G_1 \rrbracket_p )$ remains small, further relying on the triangle inequality property of $m$
to ensure the distance to the oracle does not grow considerably. Sometimes this initial choice is refined through a process known as re-training~\cite{Erwei2019}.

\begin{figure}
\begin{center}
\adjustbox{scale=1.5,center}{
\begin{tikzcd}
{{\mathbb R}^n} \arrow[r,"\llbracket G_1 \rrbracket"] & {{\mathbb R}^m}  \\
{{\mathbb F}^n} \arrow{r}[name=A]{\llbracket G_2 \rrbracket} \arrow[hookrightarrow]{u}{\iota^{n}_1} & {{\mathbb F}^m} \arrow[hookrightarrow]{u}[swap]{\iota^m_1} \\
{\mathbb X}  \arrow{r}[name=B,swap]{\left. \llbracket G_3 \rrbracket \right|_{\mathbb X}} \arrow[u,"\simeq"]{u}[swap]{\phi_i} \arrow[hookrightarrow]{d}[swap]{\iota^{kn}_2} & {\mathbb Y} \arrow[u,"\simeq"]{u}[swap]{\phi_o} \arrow[hookrightarrow]{d}{\iota^{km}_2} \\
{{\mathbb B}^{kn}} \arrow{r}[name=C,swap]{\llbracket G_3 \rrbracket} & {{\mathbb B}^{km}}
\arrow[to path={(A) node[midway,scale=1.5] {$\circlearrowleft$}  (B)}]{}
\end{tikzcd}
}
\end{center}
\caption{\label{digital-design}An abstract view of a typical digital design process. Inclusion maps are indicated by $\hookrightarrow$ and isomorphisms by $\simeq$. Starting from a specification graph $G_1$, the designer constructs a network $G_2$ operating on finite-precision datatypes, typically
fixed or floating point, as described in the text. A `synthesis tool' then automatically creates a Boolean network $G_3$, known as a `netlist'. The netlist implements
the function $\llbracket G_2 \rrbracket$ in the sense that $\llbracket G_3 \rrbracket \circ \iota^{kn}_2 \circ \phi_i^{-1} = \iota^{km}_2 \circ \phi_o^{-1} \circ \llbracket G_2 \rrbracket$. Here we distinguish ${\mathbb X}$ and ${\mathbb Y}$ from ${\mathbb B}^{kn}$ and 
${\mathbb B}^{km}$ because the inclusions are often not surjective, giving rise to the well-studied problem of `Boolean don't-cares'~\cite{Brayton1990}. The lower two sections of this diagram therefore commute, while the top section `approximately commutes'.}
\end{figure}

The second step of the process is to convert the finite-precision network to a Boolean network for implementation. This process is fully automated in modern digital design tools. Firstly, each vertex in the finite-precision graph is replaced by a Boolean network defined for that particular node's function, for a pre-defined
encoding of the elements of ${\mathbb F}$ into elements of ${\mathbb B}^k$, {\em e.g.} the IEEE floating-point storage standard~\cite{IEEE-standard} which encodes each single-precision floating point number as a $k=32$-bit vector of Boolean values; this part of the process is known in digital design as `core generation'. 
Secondly, logic synthesis tools~\cite{Brayton1990} are applied to rewrite the graph to reduce its implementation cost as a circuit. The result of this process is
a Boolean network $G_3$ which can be directly implemented as a digital logic circuit. The computation implemented by $G_3$ corresponds {\em exactly} to that implemented by $G_2$ in the sense that $\phi_o \circ \left. \llbracket G_3 \rrbracket \right|_{\mathbb X} = \llbracket G_2 \rrbracket \circ \phi_i$, where $|_{{\mathbb X}}$ denotes the restriction
of the function to the domain of $\phi_i$, {\em i.e.} the middle section of the diagram commutes. 

It can therefore be seen that in a standard digital design process, the only part of the process where an approximation is induced (the upper section of Fig.~\ref{digital-design}) is not associated with topological changes to the network, while the only part of the process where topological changes are induced (the middle section of
Fig.~\ref{digital-design}) is not associated with approximation. This observation will be of importance in the sequel.



The abstract process described in Fig.~\ref{digital-design} is illustrated for a concrete example in Figure~\ref{tech-map}. The small inset figure corresponds to the topology of $G_1$, the original specification graph, where each vertex is associated with a function ${\mathbb R}^2 \times {\mathbb R}^2 \to {\mathbb R}$ given by $(w_1, w_2; x_1, x_2) \mapsto \ReLU(w_1 x_1 + w_2 x_2)$. Fixing $w_1$ and $w_2$ to specific values, quantising the computation to a 4-bit fixed-point arithmetic, and synthesising the result produces the large main figure, corresponding to $G_3$, where each vertex is associated with a 1- or 2-input Boolean function. Clearly there are
some key differences between these networks apart from their datatypes: $G_3$ has an irregular structure compared to $G_1$ and has clusters of tightly interconnected `neighbourhoods', roughly corresponding to the Boolean networks introduced for each fixed-point arithmetic function in $G_2$. However, by maintaining the
entire design process within the same graph formalism, we can also exploit the similarities: both are directed graphs operating on typed data, with nodes which can be considered as parametric functions - for the real network the parametric functions are dot products with parameters given by weights, for the Boolean network they are 
Boolean functions with the parameter indicating which function from ${\mathcal B}_1$ or ${\mathcal B}_2$ has been selected by the logic synthesis tool.

\begin{figure}
\begin{center}
\def\stackalignment{l}
\bottominset{\fbox{\includegraphics[width=0.12\textwidth]{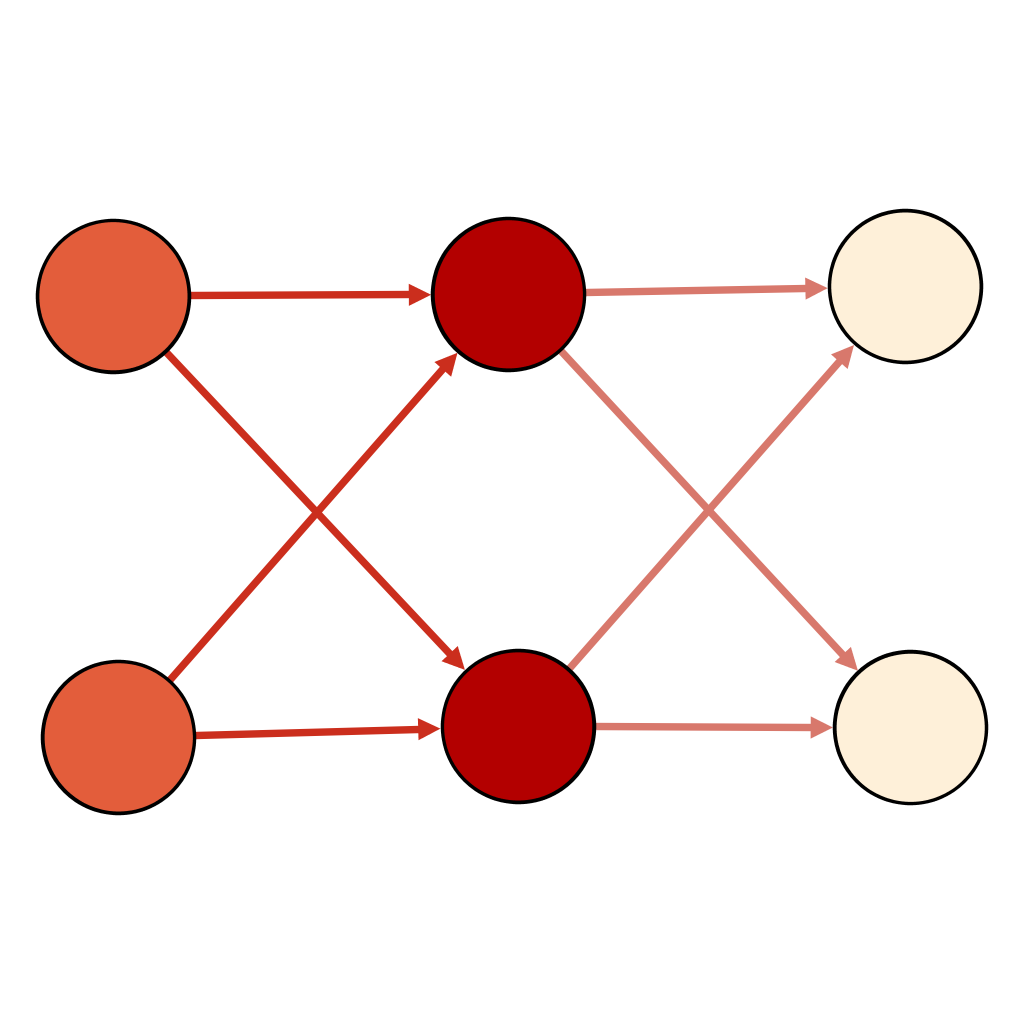}}}{\includegraphics[width=0.6\textwidth]{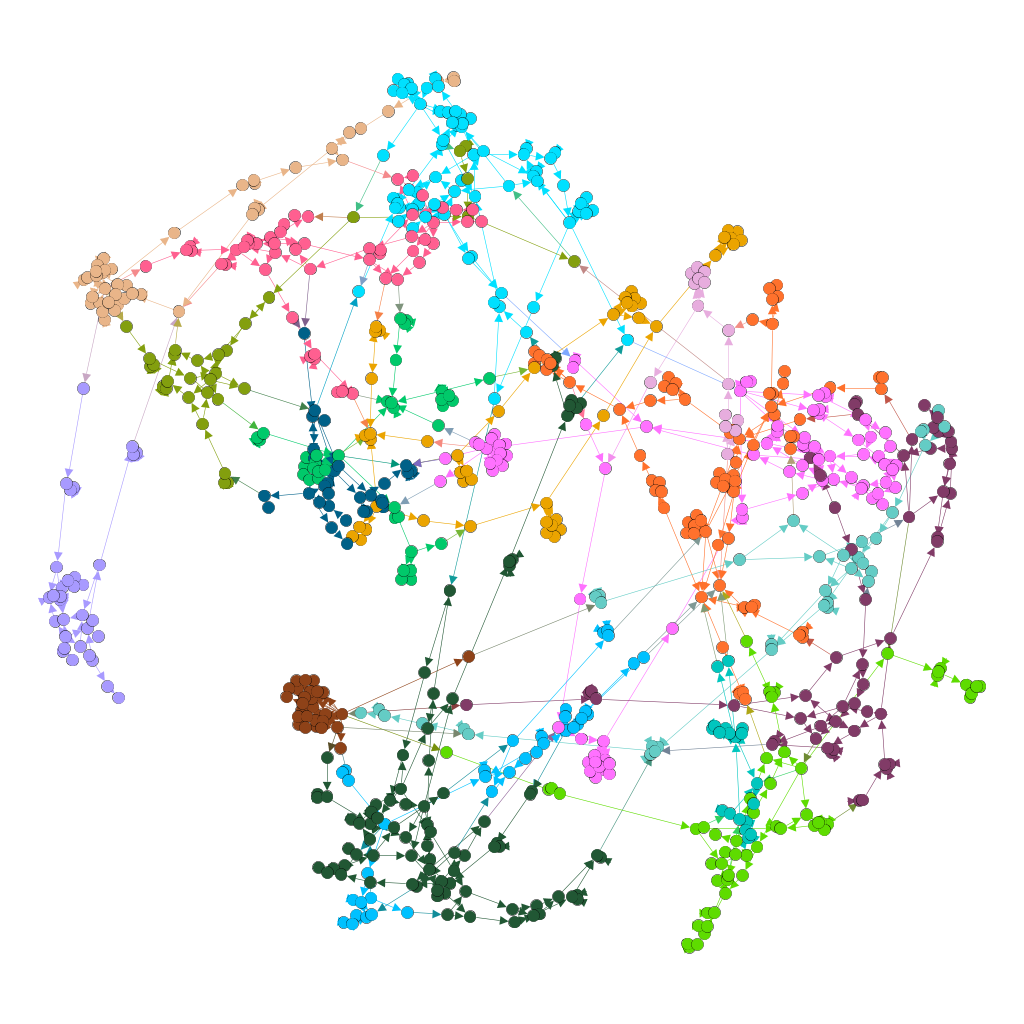}}{6pt}{6pt}
\caption{\label{tech-map}A Boolean network (main figure) corresponding to a real network (embedded figure). In the real network, nodes with inedges all correspond to a function ${\mathbb R}^2 \to {\mathbb R}$ given by $(x_1, x_2) \mapsto \ReLU(w_1 x_1 + w_2 x_2)$ for some -- possibly distinct -- parameter $w$. In the Boolean network, nodes correspond to simple logic functions from ${\mathcal B}_1$ or ${\mathcal B}_2$ produced by a synthesis tool~\cite{Yosys}, implementing a 4-bit fixed-point quantisation of the real network. Tightly interconnected regions can be seen, corresponding to the Boolean implementation of individual arithmetic operations. Rendering of both graphs is via Gephi~\cite{Gephi}, with colouring by `community'.}
\end{center}
\end{figure}

\vspace{1cm}


\noindent {\bf Binarised Neural Networks}

Driven by the desire to reduce energy consumption and improve performance as much as possible, an extreme form of fixed-point arithmetic has been used in so-called binarised neural networks (BNNs)~\cite{DBLP:journals/corr/CourbariauxB16}. In these neural networks, both the weights
and the activation signals are constrained to be drawn from $\{-1,+1\}$, resulting in extremely efficient implementations~\cite{Umuroglu2017}. 
A classical function $f : {\mathbb R}^{n+1} \times {\mathbb R}^n \to {\mathbb R}$ given by $(w, c; x) \mapsto \sigma( w^T x - c )$ implemented by component of a deep-neural network, is aggressively 
quantised to  $\textsc{bnn} : \{-1,+1\}^n \times {\mathbb Z}_n \times \{-1,+1\}^n \to \{-1,+1\}$ given by $(w, c; x) \mapsto +1$ for $w^T x \geq c$, and $(w, c; x) \mapsto -1$ otherwise. 
The key to the implementation efficiency of such functions comes from the near-elimination of hardware-expensive multiplication operations: multiplication in a vector scalar product is reduced to a 
Boolean exclusive (XNOR) function. Meanwhile, the addition in the scalar product is reduced to calculation of Hamming weight (population count), which admits efficient implementations~\cite{Warren2012}. 


Although BNNs have received a lot of attention, the general view in the implementation community is that neural networks constructed in this way are not universally able to implement as good quality classification on complex data sets compared to more precise data representations. This observation has led to manufacturers 
including configurable finite-precision datapaths typically down to 4-bit~\cite{IMGNNA} or 8-bit~\cite{ARMML}. It is instructive to pursue an alternative view, which we shall now develop.

\begin{definition}
A set $\{f_1, \ldots, f_n\}$ of Boolean functions is {\bf functionally complete} if every Boolean function $f$ can be obtained as a finite composition of these functions\cite{Clote2002}.
\end{definition}

Through an appropriate pair of encodings $\varphi_1 : {\mathbb D} \to {\mathbb B}^n$, and $\varphi_2 : {\mathbb B}^m \to {\mathbb E}$, it therefore follows that any function between finite sets $f: {\mathbb D} \to {\mathbb E}$ can be implemented by a Boolean network using a functionally complete set of Boolean functions 
at its vertices, similarly to Fig.~\ref{digital-design}, {\em i.e.} $f = \varphi_2 \circ \llbracket G \rrbracket \circ \varphi_1$ for some Boolean network $G$.

\begin{theorem}
\label{universality}
The set of node functions in a Boolean implementation of Binarized Neural Networks is functionally complete.
\end{theorem}

\begin{proof}
We shall use the bijection $\phi : \mathbb{B} \to \{-1,+1\}$ defined by $\bot \mapsto -1$, $\top \mapsto +1$.
Clote and Kranakis~\cite{Clote2002} provide necessary and sufficient conditions for a set of Boolean functions to be functionally complete; one well-known such set is $\{\wedge, \vee, \neg\}$, together with the constants $\bot$ and $\top$. The equivalences below can easily be shown through enumeration:

\begin{equation}
\begin{split}
x \wedge y & \Leftrightarrow \phi^{-1} \circ \textsc{bnn}_{(+1,+1),+2}\left( \phi(x), \phi(y) \right) \\
x \vee y  & \Leftrightarrow \phi^{-1} \circ \textsc{bnn}_{(+1,+1),0}\left( \phi(x), \phi(y) \right) \\
\neg x & \Leftrightarrow \phi^{-1} \circ \textsc {bnn}_{(-1),+1}\left( \phi(x) \right) \\
\end{split}
\end{equation}

\end{proof}

Note that it is therefore always possible to construct a real-valued DNN which, when quantised to produce a BNN, implements any Boolean function, including those Boolean functions that would have been derived via traditional design techniques (Fig.~\ref{digital-design}) using any finite-precision datatype $\mathbb{F}$,
{\em i.e.} BNNs easily satisfy our Property~\ding{184}. The theorem therefore challenges received wisdom that binarised neural networks are not always able to produce the required accuracy on a classification task. So why this apparent discrepancy in practice? 
The issue is not with the computational generality of BNNs, but rather with the traditional design technique, which is unable
to adapt the {\em topology} of the network to the requirements of the underlying datatype.

\begin{corollary}
\label{topology-datatype}
Accuracy-optimal network topology depends on finite-precision datatype.
\end{corollary}


This corollary leads to a conjecture on future design methods for efficient neural networks, which generalises some empirical observations, {\em e.g.} that reducing precision can be compensated by increasing network depth~\cite{Venkatesh2017} or width~\cite{Su2018}.
Today, digital circuits are universally implemented using CMOS technology~\cite{Weste2010}, whether in a microprocessor or a custom circuit design. CMOS circuits form
extremely efficient implementations of nonlinear operations with a single output bit. This contrasts sharply with the standard nodes of real-valued DNNs, the inner product and the $\ReLU$, which are piecewise linear but arbitrarily precise. The usual approach to this dichotomy is to use wide enough finite-precision
datatypes to make the hardware {\em emulate the real-valued model}: but at what cost?

\begin{conjecture}
Future efficient neural network topologies will be driven by both the topology of the data and by the nature of the discrete representation of the activations. The current separation between approximation (without topological changes) and topological changes (without approximation) will not
survive the drive for efficient computation.
\end{conjecture}


\section{Boolean Networks for Lipschitz Functions}
\label{WhichFunctions}

Since we have demonstrated the link between topology and data representation in deep neural networks, a natural question arises: which topologies may form good choices for learning Boolean functions? Perhaps one may even remove the ${\mathbb F}$ level of abstraction in Fig.~\ref{digital-design}, which 
would then become equivalent to {\em learning the arithmetic}.

In Section~\ref{DeepLearning}, we discussed the properties of inference functions in a continuous setting; we shall now extend this discussion to Boolean networks. 
The aim of this section is to focus on Property~\ding{182}: how can we develop Boolean networks exhibiting good generalisation?

We explained, following Strang, the centrality of continuity to generalisation in Section~\ref{DeepLearning}. The advantage of working with Lipschitz continuity is that we can directly transfer this idea to the Boolean setting. Here, every function $f : ({\mathbb B}^n,d) \to ({\mathbb B}^m,e)$ is Lipschitz, since we may take the 
Lipschitz constant $k = \max_{(a,b) \in {\mathbb B}^m \times {\mathbb B}^m} (e(a) - e(b))$, so it is not meaningful to talk about continuity in absolute terms, but rather about the value of the Lipschitz constant. We shall therefore study the question {\em which Boolean networks give rise to $k$-Lipschitz functions?} The intuition here is that the lower the Lipchitz constant, the better the function meets the desirable property that small input perturbations cause at most small output perturbations.

Before investigating a concrete example of a simple Boolean circuit in this context, let us consider typical ways to define a metric on the Boolean vectors forming the inputs and outputs of a circuit. It will be helpful to define $\varphi : {\mathbb B} \to \{0,1\}$ as $\bot \mapsto 0$, $\top \mapsto 1$.
Although not strictly necessary, it is typical to consider metrics induced by norms of encoded data, {\em e.g.}~$d(a,b) = || \phi_i(a) - \phi_i(b) ||$ for the domain, where $\phi_i : {\mathbb B}^n \to {\mathbb R}^d$. Here we may interpret $\phi_i$ as denoting a real vector represented by the Boolean inputs. A trivial example would be $\phi_i( a_1, a_0 ) = 2\varphi(a_1) + \varphi(a_0)$, a representation of a two-bit scalar integer in standard binary arithmetic. A more complex scalar encoding corresponds to IEEE single or double-precision floating point, as explicitly given in the standard~\cite{IEEE-standard}.

It is instructive to consider the most basic typical arithmetic circuit, known as a {\em ripple-carry adder}, shown in Fig.~\ref{adder}\cite{Koren2001}. Each leaf node implements a Boolean function known as a {\em full adder}: 
$\textsc{fa}: (a,b,c_i) \mapsto (a \oplus b \oplus c, a \wedge b \vee c_i \wedge (a \vee b))$, where $\oplus$ denotes Boolean XOR.
We can consider this circuit as implementing a function $f : {\mathbb B}^n \times {\mathbb B}^n \times {\mathbb B} \to {\mathbb B}^{n+1}$. If we define $w_k : {\mathbb B}^k \to {\mathbb Z}$ as the function mapping vectors of Boolean
values to the number they represent in a standard binary integer encoding:

\begin{equation}
w_k(x) = \sum_{i=0}^{k-1}{ \varphi(x_i) 2^i },
\end{equation}

\noindent, then it can be seen why the Boolean network is referred to as an {\em adder}: $+ \circ (w_n,w_n,\varphi) = w_{n+1} \circ f$, where $+$ denotes standard integer addition. In the formalism of Fig.~\ref{digital-design}, $\phi_i = (w_n,w_n,\varphi)$, $\phi_o = w_{n+1}$.

\begin{figure}

\begin{center}
\def\maxbit{1}

\begin{subfigure}{0.6\textwidth}
\adjustbox{scale=0.6,center}{
\begin{tikzpicture}[auto, scale=2.0, >=latex']
    \foreach \x in {0,\maxbit} 
      {
      \node [input, name=c\x] (c\x) at (-\x+1.9,0.25) {};
      \node (cend\x) at (-\x+1.35,0.25) {};
      \node [name=a\x] (a\x) at (-\x+0.95,1) {$a_\x$};
      \node (aend\x) at (-\x+0.95,0.45) {};
      \node [name=b\x] (b\x) at (-\x+1.3,1) {$b_\x$};
      \node (bend\x) at (-\x+1.3,0.45) {};
      \draw (-\x+0.9,0) rectangle (-\x+1.4,0.5) {};
      \node (fa\x) at (-\x+1.1,0.25) {\textsc{FA}};
      \draw [->] (a\x) -- (aend\x);
      \draw [->] (b\x) -- (bend\x);
      \draw [->] (c\x) -- node{$c_\x$} (cend\x);
      \node [name=s\x] (s\x) at (-\x+1.125,0.05) {};
      \node (send\x) at (-\x+1.125,-0.5) {$s_\x$};
      \draw [->] (s\x) -- (send\x);
       }
     \node (dot) at (-\maxbit+0.85,0.25) {};
     \node (dotend) at (-\maxbit-0.15,0.25) {};
     \draw [dotted] (dot) -- (dotend);
    \foreach \x in {1,2} 
      {
      \node [input, name=c\x] (c\x) at (\x-3.1,0.25) {};
      \node (cend\x) at (\x-3.65,0.25) {};
      \node [name=a\x] (a\x) at (\x-4.05,1) {$a_{n-\x}$};
      \node (aend\x) at (\x-4.05,0.45) {};
      \node [name=b\x] (b\x) at (\x-3.65,1) {$b_{n-\x}$};
      \node (bend\x) at (\x-3.65,0.45) {};
      \draw (\x-4.1,0) rectangle (\x-3.6,0.5) {};
      \node (fa\x) at (\x-3.9,0.25) {\textsc{FA}};
      \draw [->] (a\x) -- (aend\x);
      \draw [->] (b\x) -- (bend\x);
      \draw [->] (c\x) -- node{$c_{n-\x}$} (cend\x);
      \node [name=s\x] (s\x) at (\x-3.875,0.05) {};
      \node (send\x) at (\x-3.875,-0.5) {$s_{n-\x}$};
      \draw [->] (s\x) -- (send\x);
      }
    \node (cn1) at (-3.05,0.25) {};
    \node (cn2) at (-3.55,0.25) {};
    \node (cn2b) at (-3.5,0.3) {};
    \draw (cn1) -- (cn2);
    \node (cn3) at (-3.5,-0.5) {$c_n$};
    \draw [->] (cn2b) -- (cn3);
\end{tikzpicture}
}
\caption{An $n$-bit adder network.}
\end{subfigure}
\begin{subfigure}{0.35\textwidth}
\adjustbox{scale=0.6,center}{
\begin{tikzpicture}[auto, scale=2.0, >=latex']
    \foreach \x in {0,\maxbit} 
      {
      \node [name=c\x] (c\x) at (-\x+1.9,0.25) {};
      \node (cend\x) at (-\x+1.35,0.25) {};
      \node [name=a\x] (a\x) at (-\x+1.1,1) {$a_\x$};
      \node (aend\x) at (-\x+1.1,0.45) {};
      \draw (-\x+0.9,0) rectangle (-\x+1.4,0.5) {};
      \node (fa\x) at (-\x+1.1,0.25) {\textsc{$f_\x$}};
      \draw [->] (a\x) -- (aend\x);
      \node [name=s\x] (s\x) at (-\x+1.125,0.05) {};
      \node (send\x) at (-\x+1.125,-0.5) {$s_\x$};
      \draw [->] (s\x) -- (send\x);
       }
    \draw [->] (cend1) -- (c1);
    \draw (cend0) -- (c0);
    \node (cn1) at (-0.05,0.25) {};
    \node (cn2) at (-0.55,0.25) {};
    \node (cn2b) at (-0.5,0.2) {};
    \draw [->] (cn2) -- (cn1);
    \node (cn3) at (-0.5,1.0) {$c$};
    \draw (cn3) -- (cn2b);
    \node (q) at (1.84,0.3) {};
    \node (qs) at (1.84,-0.5) {$q$};
    \draw [->] (q) -- (qs);
\end{tikzpicture}
}
\caption{Reversing `carry' direction.}
\end{subfigure}
\end{center}

\caption{\label{adder}Lipschitz properties and network topology.}
\end{figure}
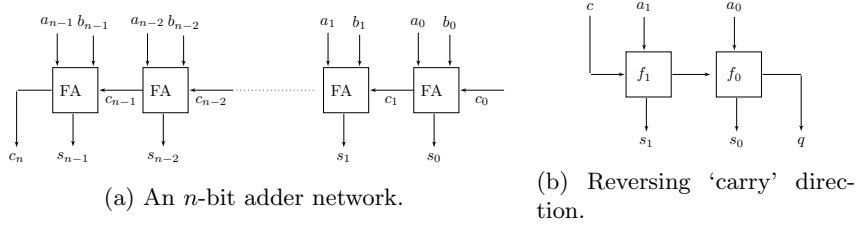

Let us equip the input and output spaces with suitable metrics, {\em e.g.}~those induced by the 1-norm of the difference in their word-level representation:

\begin{equation}
\begin{split}
d( (a, b, c), (a', b', c') ) & = \left| w_n(a) - w_n(a') \right| + \left| w_n(b) - w_n(b') \right| + \left| \varphi(c) - \varphi(c') \right| \\
e( (c, s), (c', s') ) & = \left| w_{n+1}( c, s_{n-1}, \ldots, s_0 ) - w_{n+1}( c', s'_{n-1}, \ldots, s_0 ) \right|
\end{split}
\end{equation}



\begin{lemma}
The function $f : ({\mathbb B}^n \times {\mathbb B}^n \times {\mathbb B}, d) \to ({\mathbb B}^{n+1},e)$ implemented by a ripple-carry adder is $1$-Lipschitz for any $n$.
\end{lemma}

\begin{proof}
\begin{equation}
\label{stdmetric}
\begin{split}
e( (c_n, s), (c'_n, s') ) & = \left| w_{n+1}( c_n, s_{n-1}, \ldots, s_0 ) - w_{n+1}( c'_n, s'_{n-1}, \ldots, s_0 ) \right| \\
				& = \left| w_n(a) - w_n(a') + w_n(b) - w_n(b')  + \varphi(c_0)  - \varphi(c'_0) \right| \\
				& \leq d( (a, b, c_0), (a', b', c'_0) )
\end{split}
\end{equation}
\end{proof}

How does this $1$-Lipschitz property arise? Note that from the {\em topology of the network alone}, we cannot conclude anything useful about the minimal Lipschitz constant of the function implemented; replacing the function of the leaf nodes with an alternative function $(a,b,c) \mapsto (\bot, c)$ results in a minimal Lipschitz
constant of $2^n$ rather than $1$. Changing the metrics -- equivalent in the norm-induced case to encoding the input or output with a different number system -- could equally impact on the Lipschitz properties. Finally, a different network topology based on the same full-adder leaf nodes could clearly 
lead to a different minimal Lipschitz constant. Thus the minimal Lipschitz constant exhibited by a function implemented by a network will generally depend on three things: the topology 
of the network, the leaf-node functionality, and the encoding / metrics associated with the inputs and outputs of the network. Even if we assume the latter to be fixed, the interaction between the former two features is not ideal if we wish to {\em learn} the functionality of nodes in the network: local decisions on Boolean functionality can potentially have global impact on generalisation behaviour of a network. 

Learning from the $n$-bit adder example, one natural approach to generating functions with low Lipschitz constant appears to be to reverse the direction of the `carry' edges $c_i$. If these edges are reversed, then no path exists between $a_i$, $b_i$, $c_i$ and $s_j$ or $c_j$ for
any $j > i$, meaning that changes in low-significance input bits cannot impact high-significance output bits. This topology is also appealing because it corresponds directly to most-significant-digit-first arithmetic, a universal approach to computation pioneered by Ercegovac~\cite{Trivedi1977} 
in the 1970s for computer arithmetic: through a suitable change in the encoding $w_k$, this topology can be utilised to implement all the basic arithmetic operators~\cite{Ercegovac2003}. However, such a topology does not guarantee a particular Lipchitz constant for the metrics defined in (\ref{stdmetric}), because small changes 
in the input metric can still correspond to large changes in the most-significant-digit: one sees this for example with the transition $011111 \to 100000$, a change of one but with a most-significant-digit bit flip. To avoid this issue, one must either change the encoding of the
network inputs and outputs or place restrictions on the Boolean functionality of the nodes.  The former approach -- selecting an optimal encoding of the input space as Booleans -- is an open problem. A trivial but inefficient solution would be to utilise a unary encoding. More efficient solutions could potentially draw deeply from the area of {\em combinatorial Gray codes}~\cite{Savage1997}, {\em i.e.}~methods for generating combinatorial objects 
(such inputs of a discrete-valued neural network, 
drawn from ${\mathbb X}$), so that successive objects differ by a small degree. As noted by Savage~\cite{Savage1997}, Gray codes are not preserved under bijection, and it is exactly this property that could suggest implementation-appropriate coding.

\begin{OpenProblem}
For future deep neural networks, what input and output codings are commensurate with the properties of good inference functions identified in Section~\ref{DeepLearning}, and how do they depend on the input probability space and oracle function?
\end{OpenProblem}

The author performed the following simple experiment to investigate the latter approach, {\em i.e.}~restricting Boolean functionality to ensure a certain Lipschitz constant for fixed topology and metrics. Consider the simple topology shown in Fig.~\ref{adder}(b) with associated metrics 
$d((c,a_1,a_0),(c',a_1',a_0')) = |w_3(c,a_1,a_0) - w_3(c',a_1',a_0')|$ and $e((s_1,s_0,q),(s_1',s_0',q')) = |w_3(s_1,s_0,q) - w_3(s_1,s_0,q')|$. There are $(2^4 \times 2^4)^2$ choices for the Boolean functionality of $(f_1,f_0)$. If we assume that neither constant functions nor those in ${\mathcal B}_1$ are of interest, then there are $100$ choices
for each of $f_1$ and $f_0$. A complete enumeration identifies 376 pairs of Boolean functions $f_1$, $f_0$ for which the network implements a 2-Lipschitz function. One may go further and ask whether we can identify a set of choices for the function of $f_1$ and the function of $f_0$ such
that we may {\em arbitrarily} choose functions from these two sets while maintaining the 2-Lipschitz property, effectively decoupling the choice of leaf functionality from topology. We shall refer to such sets as a `functional decoupling' for given topology, value of $k$, and metrics.

\begin{definition}
Given a network $N$ with enumerated vertices $n_i$, implementing a $k$-Lipschitz function $\llbracket N \rrbracket: ({\mathbb X},d) \to ({\mathbb Y},e)$, a {\bf functional decoupling} is a tuple of sets $S_i$ such that for every vertex, $\textsc{func}_i$ may be replaced by any element of $S_i$
independently, while maintaining the $k$-Lipschitz property.
\end{definition}

Consider a bipartite graph with node set $N_1 \cup N_0$, where $N_1$ is in one-to-one correspondence with the set of choices for function $f_1$ and $N_0$ similarly for function $f_0$, and in
which edges $\{n_1, n_0\}$ correspond to the pairs of functions resulting in a network implementing a 2-Lipschitz function. A biclique~\cite{Bondy1976} of this graph corresponds to a decoupled set. Using the algorithm of Gillis and Gilneur~\cite{Gillis2014} reveals such a biclique of 
size $(6,10)$ for this topology\footnote{Code at: {\tt https://github.com/constantinides/rethinking}}, {\em i.e.} {\em any} combination of these choices of node function results in a 2-Lipchitz network function.  

\begin{OpenProblem}
Given metrics on input and output, a Lipschitz constant $k$, and a network topology, is there a useful characterisation of exactly which functions can be implemented by a network with this topology using only leaf functions drawn from functional decouplings?
\end{OpenProblem}

The significance of this problem is that it would help us to characterise the extent to which it is useful to consider promising network topologies separately from leaf functions.

\section{The Discrete-Continuous Divide: Preliminary Work}
\label{TwoApproaches}

One of today's most promising platforms for practical realisation of very high performance deep neural networks today is the Field-Programmable Gate Array (FPGA)~\cite{Hauck2007}. These architectures provide an interesting case study for exploring some of the ideas presented in this paper, because there is a natural 
choice for the set of leaf functions implemented in a network: the set $\mathcal{B}_K$~\cite{Clote2002} of $K$-input Boolean functions, where $K$ is a device-specific parameter. This is a natural choice because the underlying architecture is actually built of small physical Boolean lookup tables, each programmable to implement
any one of the functions in $\mathcal{B}_K$, together with programmable interconnect able to connect these lookup tables in an effectively arbitrary topology ($K=6$ is common). 

Wang {\em et al.}~\cite{Wang2019b} have recently begun to explore the potential for making use of the additional flexibility provided by these lookup tables. In this initial work -- which we call LUTNet -- we begin by taking a reasonably traditional approach, following~\cite{Ghasemzadeh2018}: some standard DNN benchmarks from the 
literature are quantised to use single-bit weights from $\{-1,+1\}$, and retrained to improve classification accuracy. In the resulting network, many of the vertices have function $(w; x) \mapsto wx$, usually as part of the standard inner product common in DNNs. We observe that such computation is inefficient, because the basic lookup tables
are not being used to their full potential: in the extreme, we have hardware capable of implementing any function from $\mathcal{B}_6$ used solely to implement 2-input XNOR gates. We therefore modify the network in the following way. Firstly, we replace the vertex functions $\{-1,+1\} \times \{-1,+1\} \to \{-1,+1\}$ given by
$(w; x) \mapsto wx$ by the strictly more general class of functions $\mathbb{B}^{2^K} \times \{-1,+1\}^K \to \{-1,+1\}$ consisting of {\em all} functions (isomorphic to) $\mathcal{B}_K$, where the parameter selects the particular function. To make use of the additional support of these functions ($K$ two-valued activations rather than 
just one), we heuristically allocate the additional inputs to connect to other nodes in the network with low values of weight before quantisation. After initially setting the new functions to reproduce the original, {\em i.e.} selecting the parameters from $\mathbb{B}^{2^K}$ to be precisely those regenerating the function $(w; x) \mapsto wx$,
we then retrain the network using standard Stochastic Gradient Descent (SGD) methods. Finally, we simplify the network topology through a standard `pruning' technique~\cite{PRU_CNN_TRAIN_PRUNE_RETRAIN}. The intuition of this process is that the nonlinear generality of the class $\mathcal{B}_K$ may compensate for the pruning, resulting in a higher accuracy for a 
given number of Boolean network nodes. This is indeed what we observe in Fig.~\ref{plot:AREA_LUT_TRADEOFF}, which represents the classification error rate on the test set versus network area (in LUTs) for classification the CIFAR-10~\cite{CIFAR10} dataset containing 60,000 32x32 colour images of 10 different classes, using the CNV neural network model~\cite{Umuroglu2017} as the baseline topology
from which the modifications described above are made to the largest layer - in the case of CNV, this is a sizeable convolutional layer with 256 outputs, operating with 3x3 kernels~\cite{Wang2019b}. For this network, we see a reduction in area consumption of approximately 50\% compared to the baseline implementation operating at the same classification accuracy.

	\begin{figure}
                \centering
\adjustbox{scale=0.6,center}{
               \begin{tikzpicture}

    \begin{axis}[
		width=\columnwidth,
		height=\columnwidth,
		xlabel near ticks,
		xlabel={Area occupancy (LUTs)},
		ylabel near ticks,
		ylabel={Test error rate (\%)},
        xmin=100000,
        xmax=550000,
        error bars/y dir      = both,
        error bars/y explicit = true,
    ]
        \addplot [thick, only marks, mark=x, mark options={scale=1.5, color=red}] table [y=RPerr, x=RPLUT, y error plus=RerrbarU, y error minus=RerrbarL] {err_area.txt}; \label{plt:tradeoff_rebnet}
        \addplot [thick, only marks, mark=o, mark options={scale=1.5, color=black!25!green}] table [y=2Lerr, x=2LLUT, y error plus=2errbarU, y error minus=2errbarL] {err_area.txt}; \label{plt:tradeoff_2lutnet}
        \addplot [thick, only marks, mark=+, mark options={scale=1.5, color=blue}] table [y=4Lerr, x=4LLUT, y error plus=4errbarU, y error minus=4errbarL] {err_area.txt}; \label{plt:tradeoff_4lutnet}
        \addplot [thick, only marks, mark=*, mark options={scale=1.5, color=black!25!cyan}] table [y=6Lerr, x=6LLUT, y error plus=6errbarU, y error minus=6errbarL] {err_area.txt}; \label{plt:tradeoff_6lutnet}
        \addplot [thick, dashed] coordinates {(100000,15.55) (550000,15.55)};
	\end{axis}

\end{tikzpicture}
}
	
            	\caption{
            	    Area-accuracy tradeoff for pruned ReBNet~\cite{Ghasemzadeh2018}~(\ref{plt:tradeoff_rebnet}), 2-LUTNet~(\ref{plt:tradeoff_2lutnet}), 4-LUTNet~(\ref{plt:tradeoff_4lutnet}) and 6-LUTNet~(\ref{plt:tradeoff_6lutnet}) with the CNV network and CIFAR-10 dataset.
            	    Each point is representative of a distinct pruning threshold.
            	    The dashed line shows the baseline accuracy for unpruned ReBNet.
                }
            	\label{plot:AREA_LUT_TRADEOFF}
            \end{figure}
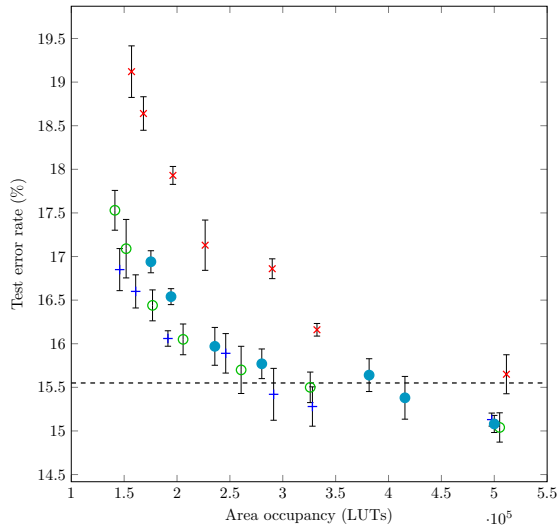

Using SGD in this discrete setting requires a {\em lifting} to a continuous interpolation, as described in detail in~\cite{Wang2019b}.  LUTNet is thus representative of one way direction in which to cross the discrete-continuous divide; some possible approaches to crossing in the opposite direction are explored in Section~\ref{Conclusions}.

\section{Future Directions}
\label{Conclusions}

It is the central thesis of this paper that there is much to learn by viewing neural networks and digital circuits as two embodiments of typed operations on graphs.

The topic of determining a good neural network topology is still in its infancy~\cite{Goodfellow-et-al-2016}. We have shown that there are additional dimensions to this problem: finite-precision
data representation and the metrics determining `closeness' of input and output also have a direct impact on efficient network topologies. Coupling these two concerns would seem to be a significant avenue for fruitful research in deep learning.

While the literature on learning appropriate parameters for predefined neural network topologies has developed rapidly in recent years~\cite{Goodfellow-et-al-2016}, systematic algorithmic approaches to learn neural network topologies from data are only recently appearing~\cite{DBLP:journals/corr/ZophL16}
and the underlying theory is limited. This mirrors the situation
in automated synthesis of digital circuits before the 1990s: the automated synthesis of logic circuits consisting of two layers (one of AND gates with optional input inversion, followed by one of OR gates) had been understood theoretically~\cite{Quine1952} and practically~\cite{Ruddell1987} before the 1990s, 
but only during that decade did the technology to optimise multi-level ('deep') Boolean networks emerge~\cite{Brayton1990}. There may be considerable scope for crossover between the electronic design automation community and the deep learning community based on this work.
Recently, there has been a resurgence of interest in the problem of exact ({\em i.e.}~optimal) 
logic synthesis~\cite{Haaswijk2018}, which -- albeit it in a different setting -- also needs to simultaneously explore topology and node functionality, and is stymied by the resulting computational complexity. This suggests a possible avenue for future development is
to lift the progress being made in this area to richer data types.

The problem of placing bounds on the number of graph nodes drawn from a certain basis set required to meet a given quality of classification, {\em e.g.}~via metric (\ref{idealmetric}), could potentially be a very interesting topic for further theoretical study. There is a rich literature on circuit complexity
bounds~\cite{Clote2002}, and it may be possible to combine these ideas with probabilistic notions from Minimum Length Descriptions~\cite{Rissanen1978} to bound minimal circuit sizes\footnote{The latter interesting suggestion originated from an anonymous reviewer of the original manuscript.}.

The $k$-Lipschitz property used in this paper is a {\em global property}, yet it may seem more natural to consider local properties. Extending the approach to networks that implement some form of locally $k$-Lipschitz functions with high probability, when the input is viewed as a random variable, may
be a fruitful way forward. In addition to reasoning about generalisation behaviour of neural networks, prior work has shown that Lipschitz continuity can play a role in regularisation of neural network models~\cite{Gouk2018}, and that minimal Lipschitz constants are hard to compute~\cite{Scaman2018}
{\em a posteriori}. These results are also suggestive that a holistic approach to topology and node functionality is appropriate, as argued in this paper.

The path seems open to investigate a variety of coding techniques for network inputs and outputs that give rise to desirable properties regarding generalisation as well as efficiency of implementation. In a different context, Dietterich and Bakiri consider distributed output coding for classification~\cite{Dietterich1995},
and it may be the case that coding theory and combinatorial enumeration approaches~\cite{Savage1997} have the potential to shed significant light on the key elements of an efficient inference function discussed in this article. 

In addition to exploring suitable classes of Boolean function, for example by attacking Open Problem 2 described in Section 4, there may be value in generalising nodes to exhibit nondeterministic behaviour. In particular, stochastic rounding has recently appeared as a promising avenue in
both training of deep neural networks~\cite{Gupta2015} and in the simulation of biologically plausible Neural models~\cite{Hopkins2019}.

Finally, we have focused entirely on deep neural networks in this article. There are, of course, many other classical machine learning techniques~\cite{Murphy2012}. We should note that once an inference algorithm for one of these classical methods has been decided upon, the algorithm
can typically be expressed as a network (in the sense of Section~\ref{DeepLearning}) corresponding to the data-flow graph~\cite{Nielson2010} of the algorithm. Just as LUTNet, described in Section~\ref{TwoApproaches}, uses BNNs as a starting point for re-training, it is equally possible to use this
network as a starting point for retraining or topological exploration.

\section*{Acknowledgements}
The author wishes to acknowledge Mr Erwei Wang for his help producing Fig.~\ref{tech-map} and Dr Christos Bouganis for comments on the initial draft and for first interesting me in modern deep learning. One of the anonymous reviewers of the original manuscript made some 
insightful suggestions for future work, and I would like to thank this reviewer for his/her generous suggestions.
This work was financially supported by the Engineering and Physical Sciences Research Council (EP/P010040/1), Imagination Technologies, and the Royal Academy of Engineering.


\bibliographystyle{IEEEtran}
\bibliography{Continuity_rev}

\end{document}